\documentclass{svproc}

\usepackage[hidelinks]{hyperref}      
\usepackage[final]{microtype}
\usepackage{type1cm}
\usepackage[bottom]{footmisc}
\usepackage{xcolor}
\usepackage{graphicx}
\usepackage[square,sort,comma,numbers]{natbib}
\usepackage{amsfonts}
\usepackage{amsmath}
\usepackage{amssymb}
\usepackage{url}

\newcommand{\pcecoe}[2]{\mathsf{#1}^{#2}}
\newcommand{\diff}{\mathop{}\!\mathrm{d}}
\newcommand{\mean}{\mathbb{E}}
\newcommand{\var}{\mathbb{V}}
\DeclareMathOperator{\covar}{CoV}
\newcommand{\norm}[1]{\left\lVert{#1}\right\rVert}

\makeatletter
\newcommand*{\bdiv}{
	\nonscript\mskip-\medmuskip\mkern5mu
	\mathbin{\operator@font div}\penalty900\mkern5mu
	\nonscript\mskip-\medmuskip
}
\makeatother

\begin{document}
\mainmatter
\title{Wiener Chaos in Kernel Regression:  Towards Untangling Aleatoric and Epistemic Uncertainty}

\titlerunning{Wiener Chaos in Kernel Regression}
\author{Timm Faulwasser \and Oleksii Molodchyk*}
\authorrunning{Timm Faulwasser and Oleksii Molodchyk}
\tocauthor{Timm Faulwasser and Oleksii Molodchyk}
\institute{Institute of Control Systems, Hamburg University of Technology, 20173 Hamburg, Germany (present working address) and 
Institute of Energy Systems, Energy Efficiency and Energy Economics, TU~Dortmund University, 44227 Dortmund, Germany
\email{timm.faulwasser@ieee.org, oleksii.molodchyk@tuhh.de}}

\maketitle
\begin{abstract}
Gaussian Processes (GPs) are a versatile method that enables different approaches towards learning for dynamics and control. Gaussianity assumptions appear in two dimensions in GPs: The positive semi-definite kernel of the underlying reproducing kernel Hilbert space is used to construct the co-variance of a Gaussian distribution over functions, while measurement noise (i.e. data corruption) is usually modeled as i.i.d. additive Gaussians.
In this note, we generalize the setting and consider kernel ridge regression with additive i.i.d. non-Gaussian measurement noise. To apply the usual kernel trick, we rely on the representation of the uncertainty via polynomial chaos expansions, which are series expansions for random variables of finite variance introduced by Norbert Wiener. We derive and discuss the analytic $\mathcal{L}^2$ solution to the arising \emph{Wiener kernel regression}. Considering a polynomial dynamic system as a numerical example, we show that our approach allows us to distinguish the uncertainty that stems from the noise in the data samples from the total uncertainty encoded in the GP posterior distribution.

\end{abstract}

\setcounter{footnote}{0}
\section{Introduction}
Gaussian Processes (GPs) are a widely considered non-parametric supervised learning method used for classification and regression. At the core of GPs are Gaussian distributions over functions with continuous domains. In the context of learning for dynamics and control, early successful application of GPs, e.g.,  to model dynamic systems can be traced back at least to \cite{Murray99a}. Early application in predictive control is considered by \cite{Kocijan03a}, while \cite{Deisenroth10a} discusses the use of GPs in reinforcement learning; see also \cite{RasmussenW06} for a tutorial reference on GPs. Due to space constraints, we do not provide a detailed overview of existing results but rather refer to the recent papers \cite{Rose2023,Capone2024,Lederer19a} and references therein for further details.

A particularly appealing feature of GPs is the posterior co-variance estimate of the predictions. This estimate is, e.g., of interest in Bayesian optimization~\cite{Heo12a} as it allows to approach the exploration-exploitation trade-off. Moreover, the posterior co-variance estimate is of relevance whenever GPs are trained to predict exogenous disturbances, e.g., energy consumption or renewable generation~\cite{Heo12a,van18a,Drgovna20a,Chen13a}.

On a formal level, the co-variance estimate of GPs is closely related to the underlying kernel and the Gaussianity assumption made on the measurement noise~\cite{Kanagawa18a}. Indeed there exists a close relation between Reproducing Kernel Hilbert Spaces (RKHS) and GPs, see, e.g., the treatments~\cite{Kanagawa18a,RasmussenW06} and the classic reference~\cite{Kailath71a}. We refer to \cite{Aronszajn1950, Berlinet2004} for an introduction to RKHS. 

In the present paper, we take a fresh look at kernel regression problems with noise-corrupted data samples, whereby the measurement noise does not need to be Gaussian. Specifically, we combine Polynomial Chaos Expansions (PCE), which originated in the most-cited journal paper of Wiener~\cite{Wiener38}, with kernel regression problems. Recently, PCE has seen use in data-driven control of linear systems \cite{tudo:pan23a,tudo:faulwasser23a}. While the link between PCE and the RKHS underlying GPs settings has, e.g., been investigated by \cite{Yan18a}, a numerical comparison of both approaches as function approximators is conducted by \cite{Gratiet16a}.  However, the previous two papers as well as the work in \cite{Torre19a} do not explore the potential of considering PCE-based measurement noise models in regression problems with arbitrary positive semi-definite kernel functions for dynamics and control.

The contributions and the outline of the present paper are as follows:
After a concise problem statement in Section~\ref{sec:problem}, we consider an intrusive uncertainty quantification approach for kernel regression problems in Section~\ref{sec:main_results}. Put differently, we use the PCE framework to quantify the effect of non-Gaussian measurement noise contained in the data samples on the prediction obtained via kernel ridge regression. Moreover, we discuss the relation of this uncertainty estimate to the classic GP posterior co-variance estimate which measures epistemic uncertainty \cite{hullermeier2021aleatoric}. 
In Section~\ref{sec:example}, we draw upon a numerical example to illustrate how our approach allows to untangle the uncertainty caused by the noisy data samples from the uncertainty related to insufficient exploration of the hypothesis space. The paper ends with conclusions and outlook in Section~\ref{sec:conclusion}.

\section{Problem Statement and Preliminaries} \label{sec:problem}
Uncertainty can be described by random variables which are elements of a probability space $(\Omega, \mathcal{F}, \mathbb{P})$. This triple comprises the set of outcomes $\Omega$, the sigma-algebra $\mathcal{F}$ over $\Omega$, and the probability measure $\mathbb{P}: \mathcal{F} \to [0,1]$. We consider a random variable $V \in (\Omega, \mathcal{F}, \mathbb{P}; \mathcal{V})$, which takes realizations in $\mathcal{V}$, as a measurable function $V: \Omega \to \mathcal{V}$. Hence, the realization of $V$ associated with the outcome $\omega \in \Omega$ is denoted as $V(\omega) \in \mathcal{V}$.
Moreover, the probability space $\mathcal L^2(\Omega, \mathcal F, \mathbb{P}; \mathcal V)$ contains random variables with finite expectation and finite variance such that $\mathcal L^2(\Omega, \mathcal F, \mathbb{P};\mathcal{V}) \subset (\Omega, \mathcal F, \mathbb{P}; \mathcal{V})$ and $V \in \mathcal L^2(\Omega, \mathcal F, \mathbb{P}; \mathcal{V})$ implies that $V:\Omega \to \mathcal{V}$ has finite $\mathcal L^2$ norm.
For an in-depth discussion on probability spaces, see \cite{Sullivan2015}.
\vspace*{-2mm}

\paragraph{The weight-space view of GPs}
Following the classic exposition of Rasmussen and Williams~\cite{RasmussenW06} we briefly recall the weight-space view on GPs. 
Consider the  data 
\[
	\mathcal D = \left\{(x_i, y_i) \mid i \in \{1, \dots, D\}  \right\}
\] 
obtained from sampling the unknown function 
$f: x \in \mathbb R^{n_x} \mapsto y \in \mathbb R^{n_y}$. For the sake of simplified exposition and to limit the notation overhead, we henceforth consider the setting with $\dim y = n_y \doteq 1$. Moreover, we use the shorthands $\mathbf x \doteq \left\{x_1, x_2, \dots x_D \right\}$, $\mathbf y \doteq \begin{bmatrix}y_1 & y_2 & \dots & y_D \end{bmatrix}^\top$, and $\mathbb I_D \doteq \left\{ 1, \dots, D \right\}$ for any $D \in \mathbb N$.
The data $\mathcal D$ is obtained via
\begin{equation} \label{eq:GP_base}
	y_i \doteq f(x_i) +  M_i(\omega), \quad i \in \mathbb I_D,~M_i \sim \mathcal N(0, \sigma_M^2),
\end{equation}
i.e., from measurements of the output of the unknown function $f$ corrupted by realizations $M_i(\omega) \in \mathbb{R}$ of independent and identically distributed (i.i.d.) white Gaussian noise of known distribution $\mathcal N(0, \sigma_M^2)$. 
Suppose that the function $f$ is of the form $f(x) = w^\top \phi(x)$ with known feature map $\phi:\mathbb R^{n_x} \to \mathbb R^{n_\phi}$ and unknown weights $w \in \mathbb{R}^{n_\phi}$. Initially, GPs treat the unknown vector $w$ as a random variable $W$ with i.i.d. entries $W_j \sim \mathcal N(0, \sigma_W^2), \forall \, j \in \mathbb I_{n_\phi}$. After receiving the data $\mathcal{D}$, this prior distribution of $W$ is updated. As a result, for each $x \in \mathbb{R}^{n_x}$, the learned GP predicts the value $f(x)$ as a (Gaussian) random variable $Y(x) \doteq W^\top \phi(x) \sim \mathcal N(\mu(x), \sigma_\mathrm{GP}^2(x))$ with
\begin{subequations} 
	\label{eq:GP_pred}
	\begin{align}
		\mu(x) &= \mathbf k(x)^\top(\mathbf K + \sigma_M^2 \mathbf I)^{-1} \mathbf y \label{eq:GP_pred_mean} \\
		\sigma_\mathrm{GP}^2(x) &= k({x}, {x}) - \mathbf k(x)^\top(\mathbf K + \sigma_M^2 \mathbf I)^{-1} \mathbf k(x), \label{eq:GP_pred_var}
	\end{align}
\end{subequations}
where $k(x, x^\prime) \doteq \sigma_W^2 \cdot \phi(x)^\top\phi(x^\prime)$ is the evaluation of the kernel function for any pair $(x, x^\prime) \in \mathbb R^{n_x} \times \mathbb R^{n_x}$, $\mathbf k(x) \doteq \begin{bmatrix} k(x_1, x) & k(x_2, x) & \dots & k(x_D, x) \end{bmatrix}^\top$ is the column vector obtained by evaluating the kernel on the entire data $\mathbf x $, and $\mathbf K$ is the square Gram matrix with $\mathbf K_{ij} \doteq k(x_i, x_j)$ evaluated for all pairs of samples in $\mathbf x$. Here, $\mathbf I$ denotes the identity matrix of appropriate dimensions. Notice, since the noise is i.i.d., the random variable $\Tilde{Y}(x) = Y(x) + M$ follows $\mathcal{N}(\mu(x), \Tilde{\sigma}^2_\mathrm{GP}(x))$, with
\begin{equation}
    \label{eq:GP_pred_var_noise}
    \Tilde{\sigma}^2_\mathrm{GP}(x) = \sigma^2_\mathrm{GP}(x) + \sigma_M^2.
\end{equation}
The close relation of the mean predictor \eqref{eq:GP_pred_mean} to kernel ridge regression and to the representer theorem \cite{Kimeldorf70a,Schoelkopf01a} is well-known~\cite{RasmussenW06,Kanagawa18a}. Indeed, the mean \eqref{eq:GP_pred_mean} of the GP posterior is always within the RKHS specified by the kernel $k(\cdot,\cdot)$ since it is a linear combination of kernel evaluations centered at datapoints $\mathbf x$, cf.~\cite[Theorem~3]{Berlinet2004}. Moreover, \eqref{eq:GP_pred_mean} coincides with the solution to the kernel ridge regression when the ridge  parameter is set to the noise variance $\sigma_M^2$.
\vspace*{-2mm}

\paragraph{Problem statement}
We consider the non-Gaussian extension to \eqref{eq:GP_base} which reads
\begin{equation} 
	\label{eq:WK_base}
	y_i \doteq f(x_i) +  M_i(\omega), \quad i \in \mathbb I_D,~M_i \in \mathcal L^2(\Omega, \mathcal F, \mathbb{P}; \mathbb R), 
\end{equation}
i.e., the additive noise is modeled as an i.i.d. random variable $M_i\in \mathcal L^2(\Omega, \mathcal F, \mathbb{P}; \mathbb R)$ of known---not necessarily Gaussian---distribution characterized by the probability measure $\mathbb{P}$. 
Given the data set $\mathcal D$ obtained from \eqref{eq:WK_base} and supposing knowledge of the distribution of the additive measurement noise $M \in \mathcal L^2(\Omega, \mathcal F, \mathbb{P}; \mathbb R)$, we are interested in deriving an extension to the classic GP predictor \eqref{eq:GP_pred} which allows for non-Gaussian noise on the measurements. To this end, we exploit the Hilbert space structure of $\mathcal L^2(\Omega, \mathcal F, \mathbb{P}; \mathbb R)$ equipped with a suitable inner product $\langle \cdot, \cdot \rangle_{\mathcal L^2}$. \vspace*{-2mm}
\paragraph{Fundamentals of polynomial chaos expansions}
PCE is an established method for uncertainty quantification and it dates back to Norbert Wiener \cite{Wiener38}. It is based on the fact that any Hilbert space can be spanned by a polynomial basis, see \cite{Sullivan2015} for a general introduction. PCE is frequently considered in systems and control \cite{Kim2013,Mesbah14a,Paulson14a}; we refer to \cite{tudo:faulwasser23a,tudo:pan23a} for applications in data-driven control.
Consider an \emph{orthonormal} polynomial basis $\{\varphi^j(\xi)\}_{j=0}^\infty$ with some argument $\xi\in \mathcal L^2(\Omega, \mathcal F, \mathbb{P}; \mathbb R)$. The basis leads to the orthogonality relation 
\[
	\langle \varphi^i(\xi), \varphi^j(\xi) \rangle_{\mathcal L^2} \doteq \int_{\omega \in \Omega} \varphi^i(\xi(\omega))\varphi^j(\xi(\omega)) \diff \mathbb{P}(\omega) = \delta^{ij},
\]
where $\delta^{ij}$ is the Kronecker delta.
As $\{\varphi^j(\xi)\}_{j=0}^\infty$ spans the probability space $\mathcal L^2(\Omega, \mathcal F, \mathbb{P}; \mathbb R)$,
any $V\in \mathcal L^2(\Omega, \mathcal F, \mathbb{P}; \mathbb R)$ can be expressed via the series expansion
\[
	V = \sum_{j=0}^{\infty}\pcecoe{v}{j} \varphi^j(\xi)\quad \text{with}\quad \pcecoe{v}{j} = \langle V, \varphi^j(\xi) \rangle_{\mathcal L^2},
\]
where $\pcecoe{v}{j} \in \mathbb R$ is called the $j$-th PCE coefficient. It is customary when working with PCE to consider the first polynomial to be $\varphi^0 =1$ and that all other polynomials $\varphi^j(\xi), j \not = 0$ have zero mean. Indeed, depending on the considered distribution of $V$, one adapts the algebraic structure of the basis and its argument $\xi$ \cite{Xiu02,Kim2013}.

From a computational point of view, infinite PCE series are impractical, and hence only finitely many terms are usually considered. This, however, may induce truncation errors \cite{kit:muehlpfordt18a}.
It is well-known that random variables following some widely-used distributions admit exact finite-dimensional PCEs (i.e., with finitely many series terms) in suitable polynomial bases, see \cite{Xiu02}.
\begin{definition}[Exact PCE representation {\cite{kit:muehlpfordt18a}}] \label{def:exact_pce}
	The PCE of a  random variable $V\in \mathcal L^2(\Omega, \mathcal F, \mathbb{P}; \mathbb R)$ is said to be exact  with dimension $L \in \mathbb N$ if \\
	$V - {\sum_{j=0}^{L-1}} \pcecoe{v}{j}\varphi^j(\xi) = 0$.
\end{definition}
Given an exact PCE of finite dimension $L$, the expected value, the variance, and the co-variance of entries in a vector-valued random variable $V\in\mathcal L^2(\Omega, \mathcal F, \mathbb{P}; \mathbb R^{n_v})$ with $n_v \in \mathbb N$ can be efficiently calculated from the PCE coefficients via
\begin{equation}\label{eq:PCEmoments}
	\mean\big[V\big] = \pcecoe{v}{0}, \quad \var \big[V\big] = \sum_{j=1}^{L-1} (\pcecoe{v}{j})^2,
	\quad  \covar\big[ V \big] = \sum_{j=1}^{L-1} \pcecoe{v}{j}\pcecoe{v}{j\top},
\end{equation}
where $\pcecoe{v}{j} \in \mathbb R^{n_v}$ and $(\pcecoe{v}{j})^2 \doteq \pcecoe{v}{j} \circ \pcecoe{v}{j}$ is the Hadamard product \cite{Sullivan2015}.
Note that any $V \in \mathcal{L}^2(\Omega, \mathcal F, \mathbb{P}; \mathcal V)$ admits a two-term exact PCE representation $V = \pcecoe{v}{0} \varphi^0(\xi) + \pcecoe{v}{1} \varphi^1(\xi)$ with $\xi = (V - \mean[V]) / \sqrt{\var[V]}$, $\varphi^0(\xi) = 1$, and $\varphi^1(\xi) = \xi$, whereas $\pcecoe{v}{0} = \mean[V]$ and $\pcecoe{v}{1} = \sqrt{\var[V]}$.

\section{Main Results} \label{sec:main_results}
Next, we consider the PCE approach to tackle the regression problem derived from the non-Gaussian setting \eqref{eq:WK_base}. To this end, we suppose that the i.i.d. additive measurement noise $M_i$ entails an exact PCE of dimension $L$, i.e., 
$M_i = \sum_{j=0}^{L-1} \pcecoe{m}{j}\varphi^j(\xi_i) \in \mathcal L^2(\Omega, \mathcal F, \mathbb{P}; \mathbb R)$.
Note that for all $i \in \mathbb I_D$, the PCE coefficients $\pcecoe{m}{j}$ are identical due to the i.i.d.-ness of the random variables $M_i$.\vspace*{-2mm}
\paragraph{Regression with $\mathcal{L}^2$ noise description}
Similar to usual kernel regression problems, we consider the linear ansatz via feature map $\phi:\mathbb R^{n_x} \to \mathbb R^{n_\phi}$, i.e., we presume the structure $f(x) = w^\top \phi(x)$. 
With the additive noise in $\mathcal L^2(\Omega, \mathcal F, \mathbb{P}; \mathbb R)$, we construct the following least-squares ($\ell^2$) loss regression problem:
\[
	\min_{w \in \mathbb{R}^{n_\phi}}  \sum_{i=1}^D \left( y_i - M_i(\omega) - w^\top \phi(x_i)\right)^2 +\rho^2 \cdot \norm{w}^2.
\]
The problem above includes the quadratic regularization term $\rho^2 \cdot \norm{w}^2$ with the Euclidean norm $\norm{\cdot}$ weighted by some constant ridge parameter $\rho^2 > 0$. Note that for each data sample, the $\ell^2$ loss includes the unknown noise realizations $M_i(\omega) \in \mathbb R$. One possible workaround would be to neglect the realizations of the measurement noise and to work with the corrupted measurements $y_i$. Alternatively, we consider the random variable description of $M_i \in \mathcal L^2(\Omega, \mathcal F, \mathbb{P}; \mathbb R)$. Rewriting the objective in the expected value sense of statistical learning yields
\begin{equation} \label{eq:L2_reg}
	\min_{W \in \mathcal L^2(\Omega, \mathcal F, \mathbb{P}; \mathbb R^{n_\phi})} \; \sum_{i=1}^D \mean\left[\left( y_i - M_i - W^\top \phi(x_i)\right)^2\right] + \rho^2 \cdot \mean\left[\norm{W}^2\right].
\end{equation}
Recall that the data $\mathcal D$ consists of real-valued tuples $(x_i, y_i)$ while in the regression problem \eqref{eq:L2_reg}, the measurement noise $M_i, i\in \mathbb{I}_D$ is considered as i.i.d. $\mathcal{L}^2$ random variables. Hence, the weights $W$ are also lifted to this space.
It is easy to see that if one does lift the decision variables $W$ to $\mathcal L^2(\Omega, \mathcal F, \mathbb{P}; \mathbb R^{n_\phi})$, then in the first-order optimality condition, the realization of $W$ matches the measurement uncertainty $M_i$ to reduce the expected loss.
We refer to \cite{Bienstock14a} for a similar discussion of equality constraints in stochastic optimization problems for energy systems.
\vspace*{-2mm}

\paragraph{Wiener kernel regression}
To avoid the technicalities of solving the regression in the infinite-dimensional $\mathcal{L}^2$ space directly, we use the PCE representation of the noise $M_i, i \in \mathbb{I}_D$ and of the weights $W$. In particular, the weights $W = \begin{bmatrix}
W_1& \dots & W_{n_\phi}\end{bmatrix}^\top$ are component-wise presented in the joint scalar basis $\{\psi^j\}_{j=0}^{L_W - 1}$ with 
\[
	\psi^0 \doteq \varphi^0 = 1~\text{and}~
	\psi^j \doteq \varphi^{\left[(j-1) \bmod (L-1)\right] + 1}(\xi_{\left[(j-1) \bdiv (L-1)\right] + 1}), \quad \forall \, j \in \mathbb I_{L_W - 1}.
\]
The basis is of dimension  $L_W \doteq  D\cdot(L-1)+1$; it includes the bases for all $M_i, i \in \mathbb{I}_D$. Put differently, we consider one basis polynomial $\psi^0$ for the mean and $D\cdot(L-1)$ basis polynomials for the non-mean parts of $W\in \mathcal L^2(\Omega, \mathcal F, \mathbb{P}; \mathbb R^{n_\phi})$ corresponding to each data sample $i$. We note that the i.i.d. property of $M_i, i \in \mathbb I_D$ is reflected in the fact that the bases for $M_i$ and $M_k$ differ in their arguments $\xi_i$ and $\xi_k$. This way, two random variables of the same distribution---and hence living in the same $\mathcal L^2$ space---can be distinguished in their PCE representations, while their PCE coefficients are equal.
Using the joint basis $\{\psi^j\}_{j=0}^{L_W - 1}$ and keeping the ridge parameter $\rho^2 > 0$ from \eqref{eq:L2_reg}, we arrive at the PCE-reformulated problem
\begin{equation} 
	\label{eq:PCE_reg}
	\begin{split}
		\min_{\substack{\pcecoe{w}{j} \in \mathbb R^{n_\phi}, \\ j \in \{0,\dots, L_W-1\}}}  \; &\sum_{i=1}^D \mean\left[\left( y_i - \textstyle\sum\limits_{j=0}^{L-1} \pcecoe{m}{j}\varphi^j(\xi_i) - \textstyle\sum\limits_{j=0}^{L_W-1} \phi(x_i)^\top \pcecoe{w}{j} \psi^j\right)^2\right] 
		\\&+ \rho^2 \cdot \mean\left[\norm{\textstyle\sum\limits_{j=0}^{L_W-1} \pcecoe{w}{j} \psi^j}^2
		\right],
	\end{split}
\end{equation}
where $W \doteq \sum_{j=0}^{L_W-1} \pcecoe{w}{j}\psi^j$ is the PCE representation of $W\in \mathcal L^2(\Omega, \mathcal F, \mathbb{P}; \mathbb R^{n_\phi})$ with vectorized PCE coefficients $\pcecoe{w}{j} \in \mathbb R^{n_\phi}$ and scalar basis functions $\psi^j \in \mathcal L^2(\Omega, \mathcal F, \mathbb{P}; \mathbb R)$ which we substitute into \eqref{eq:L2_reg}. We refer to the reformulated problem \eqref{eq:PCE_reg} as \textit{Wiener kernel regression}.
Its solution is given in the following lemma.
\begin{lemma}[Correspondence of optimization problems] \label{lem:wk}
	Suppose that the i.i.d. additive measurement noise $M_i\in \mathcal L^2(\Omega, \mathcal F, \mathbb{P}; \mathbb R), i \in \mathbb{I}_D$,  admits an exact PCE of dimension $L=2$, i.e., $M_i = \pcecoe{m}{0} + \pcecoe{m}{1} \varphi^1(\xi_i)$. Then, the following statements hold:
	\begin{enumerate}
		\item[i)] The optimal solution to \eqref{eq:L2_reg} is unique in the $\mathcal{L}^2$ sense and given by
		$W^\star = \textstyle\sum_{j=0}^{L_W-1}\pcecoe{w}{j\star}\psi^j$ with $L_W-1 =D$, whereby the PCE coefficients $\pcecoe{w}{j\star}$ are optimal in \eqref{eq:PCE_reg}.
		\vspace*{1mm}
		\item[ii)] The unique optimal solution to \eqref{eq:PCE_reg} with $\rho^2>0$ from \eqref{eq:L2_reg} is given by
		\begin{subequations}   \label{eq:WK_sol}   
			\begin{equation} \label{eq:sol_wj}
				 \pcecoe{w}{j\star} = \Phi (\mathbf K + \rho^2 \mathbf I)^{-1} \pcecoe{\mathbf{y}}{j}, \quad j \in \{0, \dots, D\},
			\end{equation}
			where $\Phi$ is a $n_\phi \times D$ matrix, whose $i$-th column is defined by $\phi(x_i)$,~$i \in \mathbb{I}_D$, $\mathbf K$ is the $D \times D$ kernel matrix from \eqref{eq:GP_pred}, and 
			\begin{equation} \label{eq:sol_yj}\pcecoe{\mathbf y}{0} =\mathbf y - \pcecoe{m}{0}  \mathbf{1}_D\quad \text{and} \quad \pcecoe{\mathbf y}{j} = -e_{j} \pcecoe{m}{1}~\text{ for } j \geq 1
			\end{equation}
		\end{subequations} 
		with $\mathbf{1}_D \doteq [1 \dots 1]^\top \in \mathbb R^D$ and $e_j$ being the $j$-th Euclidean basis vector of $\mathbb{R}^{D}$.
	\end{enumerate}
\end{lemma}
\begin{proof}
	Due to space limitations we only sketch the main steps of the proof. The assumption of an exact PCE of dimension $L=2$ for all $M_i$ gives that $L_W = D + 1$, i.e., the PCE for $W$ reads $\pcecoe{w}{0} + \sum_{j=1}^{D} \pcecoe{w}{j} \varphi^1(\xi_j)$ and its dimension corresponds to the number of available data samples plus one extra dimension for the mean $\pcecoe{w}{0}$.
	The first part of the objective in \eqref{eq:PCE_reg} with $L=2$ can be rewritten as
	\[
	\textstyle\sum\limits_{i=1}^D \mean\left[\left( 
	\left[ y_i - \pcecoe{m}{0} - \phi(x_i)^\top \pcecoe{w}{0} \right] +
	\textstyle\sum\limits_{j=1}^D \left[ - \delta^{ij} \pcecoe{m}{1} - \phi(x_i)^\top \pcecoe{w}{j} \right] \varphi^1(\xi_j)
	\right)^2\right].
	\]
	Similarly to the objective reformulations considered for stochastic data-driven optimal control by \cite{tudo:faulwasser23a}, we exploit that the measurement noise is i.i.d. and that all PCE basis functions with $j>0$ have zero mean. Next we set $\pcecoe{\mathbf y}{0}_i \doteq y_i - \pcecoe{m}{0}$ and $\pcecoe{\mathbf y}{j}_i \doteq -\delta^{ij} \pcecoe{m}{1}$ according to \eqref{eq:sol_yj}. Utilizing that $\mean[V^2] = (\mean[V])^2 + \var[V]$ for any $V \in \mathcal L^2(\Omega, \mathcal F, \mathbb{P}; \mathbb R)$ and using \eqref{eq:PCEmoments}, we rewrite the objective in \eqref{eq:PCE_reg} as
	\begin{equation*}
		\begin{split}
			\textstyle\sum\limits_{i=1}^D \left[ \pcecoe{\mathbf y}{0}_i - \phi(x_i)^\top \pcecoe{w}{0} \right]^2 + &\;\rho^2 \cdot \lVert\pcecoe{w}{0}\rVert^2 + \textstyle\sum\limits_{j=1}^D \left[ \textstyle\sum\limits_{i=1}^D \left[ \pcecoe{\mathbf y}{j}_i - \phi(x_i)^\top \pcecoe{w}{j} \right]^2 + \rho^2 \cdot \lVert\pcecoe{w}{j}\rVert^2 \right]  
		\end{split}
	\end{equation*}
	where we also exploit that \eqref{eq:PCE_reg} is an unconstrained minimization of a strictly convex quadratic function over a real-valued vector space of finite dimension, and hence there exists a unique minimizer.
	Observe that without loss of generality, we swapped the summations in the equation above. Hence the above problem can be solved for each PCE dimension $j$ individually. Indeed, for each PCE dimension $j$ it corresponds to a usual kernel regression with regularized $\ell^2$ loss. Hence we obtain \eqref{eq:sol_wj} and the computation of $\pcecoe{\mathbf y}{j}$ follows from the PCE dimension $L=2$  and the i.i.d.-ness of the random variables $M_i,~i \in \mathbb{I}_D$. This proves Assertion ii). 
	
	Assertion i) is shown similar to \cite[Prop. 1]{tudo:pan23a} via contradiction: Recall the assumption of exact PCEs for $M_i,~i \in \mathbb{I}_D$ and the corresponding joint PCE basis. Now consider an extension of the joint basis for $M_i$ and $W$. Suppose that the optimal solution admits non-zero coefficients 
	$\pcecoe{w}{j\star}$ beyond the first $L_W=D+1$ basis functions. These non-zero coefficients strictly increase the considered $\ell^2$-objective as the corresponding coefficients for $M_i:$ ~$\pcecoe{m}{j}, j>D$ are zero. \qed
\end{proof}

Under the assumptions of the above lemma, we can now give the Wiener kernel regression for the non-Gaussian setting of \eqref{eq:WK_base}.
The $\mathcal{L}^2$-optimal prediction is given by the scalar random variable $\widehat{Y}(x) \in \mathcal L^2(\Omega, \mathcal F, \mathbb{P}; \mathbb R)$ as
\begin{equation} \label{eq:WK_pred}
	\widehat{Y}(x) = \mathbf k(x)^\top(\mathbf K + \rho^2 \mathbf I)^{-1} \left[\pcecoe{\mathbf{y}}{0} + \sum_{j=1}^{D}\pcecoe{\mathbf{y}}{j} \varphi^1(\xi_j)\right].        
\end{equation}
Observe that $\widehat{Y}(x)$ is given as a linear combination of finitely many scalar basis functions $\varphi^1(\xi_j)$ of the underlying probability space $\mathcal L^2(\Omega, \mathcal F, \mathbb{P}; \mathbb R)$. Moreover, the term $\mathbf k(x)^\top(\mathbf K + \rho^2 \mathbf I)^{-1}$ is identical for all PCE dimensions. Notice that the above result easily generalizes to orthogonal (but not normal) PCE basis functions and to the case of $L\in \mathbb N$ and other convex loss functions. Notice that the PCE basis dimensions of   $W \doteq \sum_{j=0}^{L_W-1} \pcecoe{w}{j}\psi^j$ in \eqref{eq:PCE_reg} and of the estimate $\widehat{Y}(x)$ in \eqref{eq:WK_pred} are proportional to the product of $L$ and $D$. Put differently, for $L > 2$ additional terms appear in the sum in \eqref{eq:WK_pred}.
\vspace*{-2mm}

\paragraph{Quantification of aleatoric and epistemic uncertainty}
The distinction of epistemic and aleatoric uncertainties is an established concept in ML \cite{hullermeier2021aleatoric,umlauft2020real}. Specifically, \cite{hullermeier2021aleatoric} identifies the GP posterior variance $\sigma^2_\mathrm{GP}(x)$ in \eqref{eq:GP_pred_var} as epistemic uncertainty and the variance $\sigma_M^2$ of the additive measurement noise---included in \eqref{eq:GP_pred_var_noise}---as the one related to aleatoric uncertainty.
While the former variance refers to the uncertainty due to a lack of model knowledge/data, the latter refers to the uncertainty which is inherently due to randomness (measurement noise $M$) and cannot be eliminated even with infinite samples at every point in $\mathbb R^{n_x}$.
Hence, it remains to discuss the differences and commonalities of the classic GP prediction \eqref{eq:GP_pred} and our Wiener-kernel approach \eqref{eq:WK_pred}.
The relation between moments and PCE coefficients in \eqref{eq:PCEmoments} leads to the following result.
\begin{lemma}[Moments of Wiener-kernel predictors] \label{lem:WK_moments}
	Suppose that, for all $i \in \mathbb{I}_D$, the i.i.d. additive measurement noise $M_i\in \mathcal L^2(\Omega, \mathcal F, \mathbb{P}; \mathbb R)$ has an exact PCE of dimension $L=2$. Then, we have 
	\begin{subequations} \label{eq:WK_pred_mom}
		\begin{itemize}
			\item[i)] ~  \vspace*{-5mm}\begin{equation} \label{eq:WK_pred_mean}
				\mean\left[ \widehat{Y}(x)\right] = \mathbf k(x)^\top(\mathbf K + \rho^2 \mathbf I)^{-1} \mathbf y^0, \quad \mathbf y^0 = \mathbf y - \pcecoe{m}{0} \mathbf{1}_D. 
			\end{equation}  
			~\vspace*{-5mm}
			\item[ii)]  Specifically, for  $M_i\sim \mathcal N(0, \sigma_M^2)$ the choice $\rho^2=\sigma_M^2$ implies that
			the GP mean value prediction from \eqref{eq:GP_pred_mean}  is equivalent to $\mean[ \widehat{Y}(x)]$.
			\item[iii)]  Moreover, for any i.i.d. $M_i\in \mathcal L^2(\Omega, \mathcal F, \mathbb{P}; \mathbb R),~i \in \mathbb I_D$ it holds that
			\begin{equation}\label{eq:WK_pred_var}
				\var\left[ \widehat{Y}(x)\right] = \sigma_M^2 \mathbf k(x)^\top(\mathbf K + \rho^2 \mathbf I)^{-2} \mathbf k(x), \quad \sigma_M^2 = (\pcecoe{m}{1})^2.
			\end{equation}
		\end{itemize}
	\end{subequations}
\end{lemma}
\begin{proof}
	Due to space constraints, we sketch only the main steps. i) Follows directly from \eqref{eq:PCEmoments}. Notice that in the (normalized) Hermite polynomial series, a zero-mean Gaussian $\mathcal N(0, \sigma_M^2)$ can be represented via PCE with $L=2$ while the coefficients of $M_i$ are $\pcecoe{m}{0} = 0$ and $\pcecoe{m}{1} = \sigma_M$, which shows ii).
	To show iii), we apply the variance formula from \eqref{eq:PCEmoments} to \eqref{eq:WK_pred}.
	This yields
	\[
	\var\left[\widehat{Y}(x)\right]  = \mathbf k(x)^\top(\mathbf K + \rho^2 \mathbf I)^{-1}\covar\left[ \textstyle\sum\limits_{j=1}^{D}\pcecoe{\mathbf{y}}{j} \varphi^1(\xi_j)\right](\mathbf K + \rho^2 \mathbf I)^{-1} \mathbf k(x),
	\]
	which---after evaluating the above co-variance term---simplifies to \eqref{eq:WK_pred_var}. \qed
\end{proof}

The careful reader has surely recognized the difference of the classic GP variance $\sigma^2_\mathrm{GP}(x)$ from \eqref{eq:GP_pred_var} and the variance of the optimal PCE solution to the Wiener kernel regression \eqref{eq:WK_pred_var}. Notice that both expressions do not depend on the actually acquired measurements $\mathbf y$ but only on the data samples $\mathbf x$.

It is well-known that the classic GP variance \eqref{eq:GP_pred_var} corresponds to the Gaussian distribution of the next prediction conditioned on the data $\mathcal{D}$ via the co-variance encoded in the kernel $k(\cdot,\cdot)$. Recall that in \cite{hullermeier2021aleatoric}, \eqref{eq:GP_pred_var} is viewed as a proxy for epistemic uncertainty.
Indeed, a formal analysis shows that in the infinite data limit and for Lipschitz continuous kernels, $\sigma^2_\mathrm{GP}(x) \to 0$ holds \cite{Lederer19a,Kanagawa18a}. This means that \eqref{eq:GP_pred_var} vanishes with sufficiently large (and rich) data sets $\mathcal D$. 
In the case of white additive measurement noise, however, the additive structure of \eqref{eq:GP_pred_var_noise} implies that the variance $\Tilde{\sigma}^2_\mathrm{GP}(x)$ is fundamentally bounded from below by the variance of the measurement noise $\sigma_M^2$; hence, the association of $\sigma_M^2$ with aleatoric uncertainty.

An interesting observation is that despite being described as epistemic uncertainty, $\sigma^2_\mathrm{GP}(x)$ in \eqref{eq:GP_pred_var} still depends on the noise-related $\sigma_M^2$, which comes from the fact that the GP posterior \eqref{eq:GP_pred} is inferred from the noisy data \eqref{eq:GP_base}. Since the noise is assumed to be i.i.d., this, in turn, raises a question of the decomposition of \eqref{eq:GP_pred_var} into purely model-dependent and purely noise-related variance terms.
In the case of the i.i.d. Gaussian noise \eqref{eq:GP_base}, we argue that one avenue towards factoring out the noise-related term of \eqref{eq:GP_pred_var} is paved by the variance estimate
\begin{equation} \label{eq:WK_var_for_GP}
    \sigma^2_\mathrm{WK}(x) \doteq \sigma_M^2 \mathbf k(x)^\top(\mathbf K + \sigma_M^2 \mathbf I)^{-2} \mathbf k(x)
\end{equation}
derived via \eqref{eq:WK_pred_var} by setting $\rho^2 \doteq \sigma_M^2$ such that $\mathbb E[Y(x)] = \mathbb E[\widehat{Y}(x)]$, see Lemma~\ref{lem:WK_moments}.

Indeed, \eqref{eq:WK_pred_var} can be understood as the result of propagating the (not necessarily Gaussian) measurement noise $M_i \in \mathcal L^2(\Omega, \mathcal F, \mathbb{P}; \mathbb R)$ through the linear mean value predictor \eqref{eq:GP_pred_mean}. 
A rigorous convergence analysis of \eqref{eq:WK_pred_var} in view of the available data is open at this point. As a first step, the following result shows that \eqref{eq:WK_pred_var} converges to zero at any point $\Bar{x} \in \mathbb R^{n_x}$ as the number $N \in \mathbb N$ of samples \eqref{eq:WK_base} taken at $\Bar{x}$ goes to infinity. 
\begin{lemma}[Asymptotics of Wiener-kernel variance] \label{lem:WK_var_convergence}
    Suppose that after collecting $\mathcal{D}$ via \eqref{eq:WK_base}, one has  $N$ repeated samples at some $\Bar{x} \in \mathbb R^{n_x}$. Denote the resulting data set as $\mathcal{D}^\prime = \mathcal{D} \cup \Bar{\mathcal{D}}$, where $\Bar{\mathcal{D}} \doteq \{(\Bar{x} , f(\Bar{x}) +  M_{D+i}(\omega)) \in \mathbb R^{n_x} \times \mathbb R \,|\, \forall \, i \in \mathbb I_N\}$ are the $N$ repeated measurements of $f$ at $\Bar{x}$. Let $\mathbb V_N(x) \doteq \sigma^2_\mathrm{WK}(x)$ be the variance \eqref{eq:WK_pred_var} based on the data in $\mathcal{D}^\prime$ and with $\rho^2 > 0$. Then,
    \[
        \lim_{N \to \infty} \mathbb V_N(\Bar{x}) \leq  \lim_{N \to \infty} \dfrac{\sigma_M^2}{N} =0.
    \]
\end{lemma}
\begin{proof}
    For $\Bar{x} \in \mathbb R^{n_x}$, we have $\mathbb V_N(\Bar{x}) = \sigma_M^2 \mathbf k_N(\Bar{x})^\top(\mathbf K_N + \rho^2 \mathbf I)^{-2} \mathbf k_N(\Bar{x})$, where $\mathbf k_N$ and $\mathbf K_N$ contain the evaluations of the positive semi-definite kernel $k(\cdot,\cdot)$ based on $\mathcal{D}^\prime$. For the Gram matrix $\mathbf K_N$, we have
    \[
        \mathbf{K}_N = 
        \begin{bmatrix}
            \mathbf{K}\phantom{^\top} & \Bar{\mathbf{K}}\phantom{_{\Bar{\mathcal{D}}}} \\ 
            \Bar{\mathbf{K}}^\top & \mathbf{K}_{\Bar{\mathcal{D}}}
        \end{bmatrix}
    \]
    with $\Bar{\mathbf{K}} \doteq \begin{bmatrix}\mathbf{k}(\Bar{x}) & \dots & \mathbf{k}(\Bar{x})\end{bmatrix} \in \mathbb R^{D \times N}$ and $\mathbf{K}_{\Bar{\mathcal{D}}} \doteq k(\Bar{x},\Bar{x}) \cdot \mathbf{1}_{N \times N}$, whereas $\mathbf{k}_N(\Bar{x}) \doteq \begin{bmatrix}\mathbf{k}(\Bar{x})^\top & k(\Bar{x}, \Bar{x}) & \dots & k(\Bar{x}, \Bar{x}) \end{bmatrix}^\top$ $\in \mathbb R^{D + N}$. Clearly, for any $N \in \mathbb N$, the linear system $\mathbf{K}_N b = \mathbf{k}_N(\Bar{x})$ is under-determined since $b \doteq \begin{bmatrix}\mathbf{0}_D^\top & \alpha^\top \end{bmatrix}^\top$ with arbitrary $\alpha \in \mathbb R^N$,  $\alpha^\top \mathbf{1}_N = 1$, is a solution. In particular, if we set $\alpha \doteq (1/N) \cdot \mathbf{1}_N$, then $\norm{b}^2 = 1/N$ holds. This choice allows us to derive the following inequality:
    \begin{multline*}
            \sigma_M^2 \mathbf{k}_N(\Bar{x})^\top \left(\mathbf{K}_N + \rho^2 \mathbf{I} \right)^{-2} \mathbf{k}_N(\Bar{x}) = \, \sigma_M^2 \norm{\left(\mathbf{K}_N + \rho^2 \mathbf{I} \right)^{-1} \mathbf{K}_N b}^2 \leq \\ 
            \sigma_M^2 \norm{\left(\mathbf{K}_N + \rho^2 \mathbf{I} \right)^{-1} \mathbf{K}_N}^2 \norm{b}^2 = \frac{\sigma_M^2 \left(\lambda_\mathrm{max}\left[\mathbf{K}_N\right]\right)^2 \norm{b}^2}{\left(\lambda_\mathrm{max}\left[\mathbf{K}_N\right] + \rho^2\right)^2} \leq \sigma_M^2 \norm{b}^2 = \frac{\sigma_M^2}{N}.
    \end{multline*}
The matrix norm $\norm{\cdot}$ in the equation above is induced by the usual Euclidean norm. For $A \in \mathbb{R}^{n \times n}$ we have $\norm{A} = \sqrt{\lambda_\mathrm{max}\left[A^\top A\right]}$, where $\lambda_\mathrm{max}\left[\cdot\right]$ denotes the largest eigenvalue.
Recall that $\norm{Ab} \leq \norm{A} \norm{b}$  holds for any $A \in \mathbb{R}^{n \times n}$ and $b \in \mathbb{R}^n$.
Thus, with the upper bound from the above equation, we obtain $0 \leq \mathbb V_N(\Bar{x}) \leq \sigma_M^2 / N$ which proves the assertion.
\qed
\end{proof}

Lemma~\ref{lem:WK_var_convergence} shows that assuming the knowledge of the distribution of the i.i.d. noise $M_i \in \mathcal{L}^2(\Omega, \mathcal{F}, \mathbb P; \mathbb R)$, the value $f(\Bar{x})$ can be exactly determined if one takes infinitely many samples at $\Bar{x}$, i.e., $\lim_{N \to \infty}\{y_i = f(\Bar{x}) + M_i(\omega), i \in \mathbb I_N \}$. Indeed, in this case, from the central limit theorem, we have that
\[
    \lim_{N \to \infty} \frac{1}{N} \sum_{i=1}^N y_i = \lim_{N \to \infty} \frac{1}{N} \sum_{i=1}^N\left(f(\Bar{x}) + M_i(\omega) \right) = \mean \left[f(\Bar{x}) + M\right] = f(\Bar{x}) + \pcecoe{m}{0}.
\]
Thus, one can find $f(\Bar{x})$ via $f(\Bar{x}) = \lim_{N \to \infty} \frac{1}{N} \sum_{i=1}^N y_i - \pcecoe{m}{0}$, where $\pcecoe{m}{0}$ is known.
Moreover, whenever one can estimate an upper bound of $\sigma_M$, Lemma~\ref{lem:WK_var_convergence} provides an avenue to estimate the number of samples to be taken at a point $\bar x$ in order to meet a target variance $\sigma^2_\mathrm{WK}(x)$ at $\bar x$ for the learned model. 

\paragraph{Interpretation} The Wiener-kernel variance \eqref{eq:WK_var_for_GP} expresses the uncertainty in the predictions computed at a previously visited point $\bar x$ associated exclusively with the noise in the labels/targets $\mathbf{y}$ in the available data $\mathcal{D}$ (Lemma~\ref{lem:WK_var_convergence}). 
In other words, it quantifies the effect of the aleatoric uncertainty in the measurements $\mathbf y$ taken at a point $x$ on the epistemic uncertainty of the learned model evaluated at the same point $x$.
Moreover, our numerical experiments in Section~\ref{sec:example} suggest that $\sigma_\mathrm{GP}^2(x)$ in \eqref{eq:GP_pred_var} and $\sigma_\mathrm{WK}^2(x)$ in \eqref{eq:WK_var_for_GP} satisfy $\sigma_\mathrm{WK}^2(x) \leq \sigma_\mathrm{GP}^2(x)$ for all $x \in \mathbb{R}^{n_x}$. However, the formal analysis of this relation remains an open question.

\section{Numerical Example}
\label{sec:example}
We consider the following scalar discrete-time GP state space model from \cite{Beckers_2016}
\begin{equation}
	x_{k+1} = 0.01 x_k^3 - 0.2 x_k^2 + 0.2 x_k + M_k(\omega),
	\label{eq:toy_problem}
\end{equation}
where $x_k \in \mathbb R$ denotes the state at time $k \in \mathbb N$ and the additive noise $M_k \in \mathcal L^2(\Omega, \mathcal F, \mathbb{P}; \mathbb R)$ is i.i.d. First  we assume that $M_k, k \in \mathbb{I}_D$  follow the standard normal distribution $\mathcal{N}(0,1^2)$, i.e., $\pcecoe{m}{0} = 0$ and $\pcecoe{m}{1} = 1$ in the (normalized) Hermite basis. We extend to a non-Gaussian case later.
Given exact measurements of $x_k$, we are interested in predicting the next system state $x_{k+1}$ while the right-hand side function in \eqref{eq:toy_problem} is not known. We use a squared-exponential kernel $k(x,x^\prime) \doteq \sigma_f^2 \exp\left(-\lVert x - x^\prime \rVert^2 / (2 l^2)\right)$ with $\sigma_f = 4.21$ and  $l = 3.59$ obtained in \cite{Beckers_2016} through maximization of the marginal likelihood.
\vspace*{-2mm}

\paragraph{Gaussian measurement noise}
We train the above-mentioned GP using  data  $\mathcal{D}$ obtained at $N_\mathrm{x}$ different linearly-spaced $x$-locations within $[-5, 5]$. The results are depicted in Fig.~\ref{fig:uncertainties}, where data points are marked by \textit{black crosses}. The first row of plots A)-C) shows results in which the data was obtained at $N_\mathrm{x} \in\{ 2,3,5\}$ different $x$-locations. This experiment is repeated by increasing the number of samples $N_\mathrm{sam}\in\{1,5,25\}$ per $x$-location in plots A)-C), D)-F), and G)-I). 

In each of the plots, the \textit{blue solid line} indicates the true function from \eqref{eq:toy_problem} without measurement noise $M$. The \textit{green-shaded tubes} are calculated via $\mu(x) \pm \sigma_\mathrm{WK}(x)$ where $\mu(x)$ corresponds to the predicted mean (\textit{red solid line}) \eqref{eq:GP_pred_mean}/\eqref{eq:WK_pred_mean}, and $\sigma_\mathrm{WK}(x)$
is the standard deviation derived from the variance \eqref{eq:WK_var_for_GP} of the Wiener kernel regression. The GP counterpart of this uncertainty is shown using \textit{dark-gray tubes} $\mu(x) \pm \sigma_\mathrm{GP}(x)$ where $\sigma_\mathrm{GP}(x)$ is from \eqref{eq:GP_pred_var}. Finally, \textit{light-gray tubes} illustrate $\mu(x) \pm \Tilde{\sigma}_\mathrm{GP}(x)$, where $\Tilde{\sigma}^2_\mathrm{GP}(x)$ is the GP posterior variance with added measurement noise, cf.~\eqref{eq:GP_pred_var_noise}. 

We note that both of the considered GP variance tubes are larger than the green tube obtained via the Wiener kernel regression.
While the usual GP posterior variance $\sigma^2_\mathrm{GP}(x)$ quantifies the uncertainty associated both with the incomplete exploration of the hypothesis space (i.e., the set of all possible GP posterior sample functions) and with the measurement noise contained in the collected data $\mathcal{D}$, the Wiener kernel estimate $\sigma^2_\mathrm{WK}(x)$ quantifies exclusively the influence of the measurement noise. 
Specifically, in the considered numerical example, plots A), D), and G) of Fig.~\ref{fig:uncertainties} show that as the number of samples ($N_\mathrm{sam}$) increases, the noise-associated uncertainty---i.e., the green tube---shrinks (cf.~Lemma~\ref{lem:WK_var_convergence}), while the dark-gray tube of the GP posterior remains.

When five linearly-spaced locations are considered, the spacing between the neighboring $x$-locations in $\mathcal{D}$ reduces to $2.5$, which can be covered by the lengthscale $l$ of $k$. This results in a much better exploration of the hypothesis space.
As shown in plot C), the dark-gray area coincides with the noise-related green-colored uncertainty tube. As we eliminate this noise-related uncertainty by increasing the number of samples per $x$-location, we notice that our kernel regression result improves dramatically, cf. plots F) and I) in Fig.~\ref{fig:uncertainties}. However, 
the light-gray tube associated with the variance $\Tilde{\sigma}^2_\mathrm{GP}(x)$ remains quite large due to its additive structure \eqref{eq:GP_pred_var_noise}, which includes the aleatoric uncertainty $\sigma_M^2$. We observe similar behavior using other hyper-parameter values and kernels such as polynomial and exponential (results not shown here).
\vspace*{-2mm}

\begin{figure}[t]
	\centering
	\includegraphics[width=\textwidth]{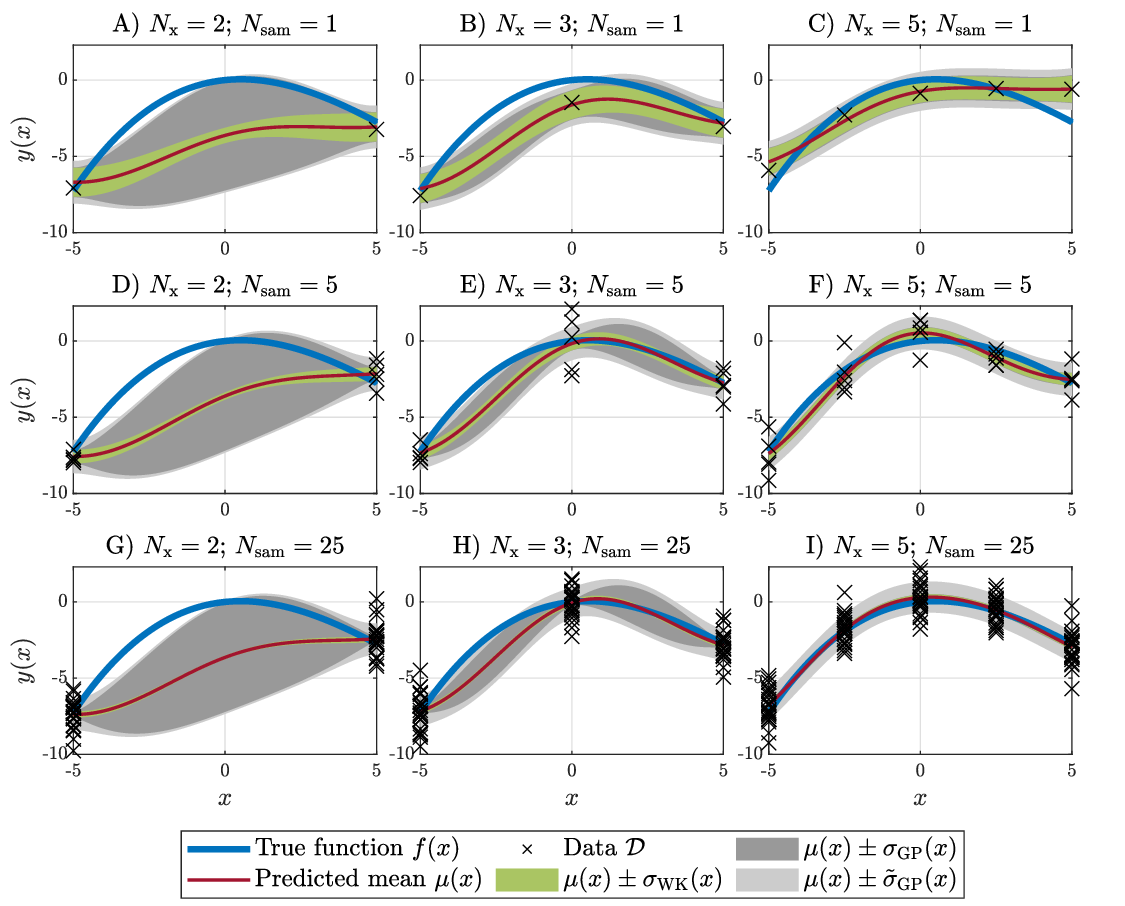}
	\caption{Comparison of the variance estimates from GP and Wiener kernel regression.}
	\label{fig:uncertainties}
\end{figure}

\paragraph{Non-Gaussian noise}
We extend \eqref{eq:toy_problem} to a non-Gaussian setting, i.e., we consider the i.i.d. noise $M_k$, $k \in \mathbb{I}_D$ to follow a $\gamma$-distribution $\mathrm{Gamma}(\alpha, \beta; x) = \left(1 / \Gamma(\alpha)\right) \beta^{-\alpha} x^{\alpha - 1} \exp(- x / \beta)$ in which $\Gamma$ denotes the Gamma function. Unlike Gaussians, this distribution allows for asymmetry, i.e., non-zero skewness. We set the parameters to $\alpha \doteq 0.25$ and $\beta \doteq 2$. Thus, the mean of this distribution is $\mu_M = \alpha \beta = 0.5$, while the variance reads $\sigma_M^2 = \alpha \beta^2 = 1$. To comply with Lemma~\ref{lem:wk}, for each $M_i$, $i \in \mathcal{I}$ with $\xi_i \sim \mathrm{Gamma}(0.25, 2)$ we choose the (orthonormal) basis of dimension $L=2$ with $\varphi^0(\xi_i) \doteq 1$ and $\varphi^1(\xi_i) \doteq (\xi_i - \mu_M) / \sigma_M$. The corresponding PCE coefficients are $\pcecoe{m}{0} \doteq \mu_M = 0.5$ and $\pcecoe{m}{1} \doteq \sigma_M = 1$.

Next we use the same kernel as before to compute $\widehat{Y}(x)$ via Wiener kernel regression \eqref{eq:WK_pred}. The linearly spaced training data $\mathcal{D}$ is chosen within $x \in [-5, 5]$ with $N_\mathrm{x} = 5$ and $N_\mathrm{sam} = 1$. We observe that the distribution of $\widehat{Y}(x)$ is modeled via a linear combination of $D = 5$ i.i.d. $\gamma$-distributed random variables. In general, it is difficult to characterize the underlying probability density function analytically. To obtain a numerical approximation, we sample the individual arguments $\xi_i$ of the series expansion in \eqref{eq:WK_pred}. This way we generate a collection of function realizations $\mathcal{C}(\widehat{Y}) = \lbrace \widehat{Y}(x;\omega) \mid \omega \in \Omega_0 \rbrace$ for some (finite) subset of outcomes $\Omega_0 \subset \Omega$. The left plot of Fig.~\ref{fig:non_gauss_dist} depicts these functions as \textit{pink solid lines} with the true function $f(x)$ (\textit{blue solid line}) and the predicted mean $\mu(x)$ (\textit{red solid line}) of the kernel regression. Using $\mathcal{C}(\widehat{Y})$, one can create a histogram and a corresponding probability density function (PDF) fit at the selected $x$-locations. In the left plot of Fig.~\ref{fig:non_gauss_dist}, these fitted PDFs are shown as \textit{red solid lines} at the $x$-locations of the data $\mathcal{D}$  using $\lvert \mathcal{C}(\widehat{Y}) \rvert = 5000$ samples. On the right-hand side of Fig.~\ref{fig:non_gauss_dist}, we project onto the plane corresponding to $x=0$. Here, we plot the histogram of $\mathcal{C}(\widehat{Y}(x=0))$ and its fitted PDF as a \textit{red solid line}. For the sake of comparison, we also plot the Gaussians $\mathcal{N}(\mu(x=0), \sigma^2_\mathrm{GP}(x=0))$ and $\mathcal{N}(\mu(x=0), \sigma^2_\mathrm{WK}(x=0))$, which would define the uncertainty tubes stemming from the GP and the Wiener kernel regression, respectively. Similarly to the plots in Fig.~\ref{fig:uncertainties}, we also present the distribution $\mathcal{N}(\mu(x=0), \tilde \sigma^2_\mathrm{GP}(x=0))$.
As seen from the right-hand side of Fig.~\ref{fig:non_gauss_dist}, the fitted distribution of $\widehat{Y}(x)$ is asymmetric. That is, the skewed measurement noise can be captured and quantified via \eqref{eq:WK_pred_var}.
\begin{figure}[t]
	\centering
	\includegraphics[width=\textwidth]{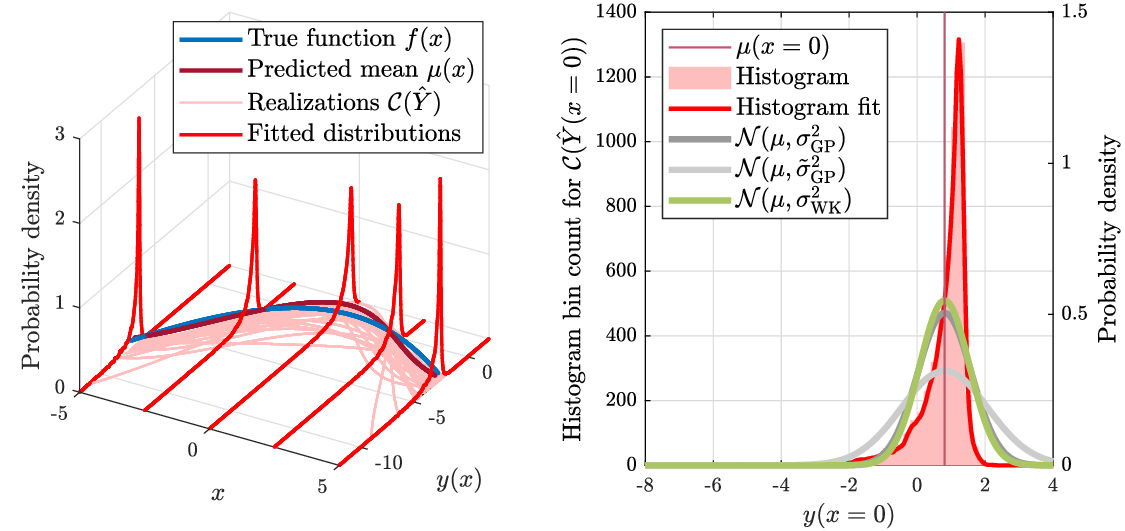}
	\caption{Example \eqref{eq:toy_problem} with $\gamma$-distributed measurement noise.}
	\label{fig:non_gauss_dist}
\end{figure}

\section{Conclusions and Outlook} \label{sec:conclusion}
The present paper discussed kernel ridge regression with training data corrupted by i.i.d. additive non-Gaussian noise of finite variance. 
Relying on the framework of Polynomial Chaos Expansions (PCE), we introduced a novel framework for Wiener kernel regression which handles the non-Gaussian measurement noise effectively. 
When the ridge parameter value is set to the variance of the i.i.d noise, the mean predictions of standard GPs and the proposed Wiener kernel regression coincide. 
The variance predictions differ, however, in their meaning. While the posterior GP variance is a measure of total epistemic uncertainty caused both by the insufficient exploration of the hypothesis space and due to the noisy data, the Wiener kernel variance only measures the effect of the (aleatoric) noise in the data samples on the prediction. We illustrate this uncertainty decomposition via numerical examples where we compare both of the variance estimates.           

Our preliminary results point towards further research. As for the presented Wiener kernel regression, there is a need for investigations regarding the properties of the derived variance \eqref{eq:WK_pred_var} concerning its analytic relation to \eqref{eq:GP_pred_var}. The interpretation of \eqref{eq:WK_pred_var} in the context of computational uncertainty \cite{Wenger2023} remains to be done. While this work considered the weight space view on GPs, the extension to the function space view is still an open problem. 
Tailored methods for hyperparameter tuning under non-Gaussian noise and the consideration of noise corruption on the sample locations $\mathbf x$ are another question for future research. The link between our results and results available for additive sub-Gaussian noise \cite{Chen2024} has to be analyzed. Finally, the proposed approach might provide new avenues for handling the exploration-exploitation trade-off in Bayesian optimization under non-Gaussian noise.

\section*{Acknowledgment}
The authors thank Ruchuan Ou for helpful feedback on the  implementation.

\bibliographystyle{spmpsci}
\bibliography{main}

\end{document}